\documentclass[letterpaper, 10 pt, conference]{ieeeconf}  

\IEEEoverridecommandlockouts                              
\pdfoutput=1

\usepackage{multicol}
\usepackage{times}

\usepackage[font=small]{caption}
\usepackage{lipsum}
\usepackage{amsmath, amssymb}
\usepackage{graphicx}  
\usepackage{ltablex, multirow}
\usepackage{algorithm}
\usepackage[noend]{algpseudocode}
\usepackage{subcaption} 
\usepackage{cleveref}
\usepackage{amsfonts,bbold}
\usepackage{url}
\usepackage{amsthm}
\usepackage{booktabs}
\let\labelindent\relax
\usepackage{enumitem}

\DeclareMathOperator*{\argmin}{arg\,min}

\usepackage{color}

\newtheorem{thm}{Theorem}[section]

\newcommand{\xpi}[1]{#1_{p_i}}
\newcommand{\xhj}[1]{#1_{h_j}}

\newcommand{\tim}[1]{\textrm{#1}}
\newcommand{\mc}[1]{\mathcal{#1}}
\newcommand{\bb}[1]{\mathbb{#1}}
\newcommand{\sota}{state-of-the-art\ }

\newcommand{\p}{p_i}
\newcommand{\h}{h_j}
\newcommand\Tstrut{\rule{0pt}{2ex}}         
\begin{document}
\title{RoadTrack: Realtime Tracking of Road Agents in Dense and Heterogeneous Environments }
\author{Rohan Chandra$^1$, Uttaran Bhattacharya$^1$, Tanmay Randhavane$^2$, Aniket Bera$^2$, and Dinesh Manocha$^1$\\
$^1$University of Maryland, $^2$University of North Carolina\\
\small{Supplementary Material at \url{https://gamma.umd.edu/ad/roadtrack}}
\vspace{-15pt}
}


\maketitle
\thispagestyle{empty}
\pagestyle{empty}
\begin{abstract}
We present a realtime tracking algorithm, RoadTrack, to track heterogeneous road-agents in dense traffic videos. Our approach is designed for dense traffic scenarios that consist of different road-agents such as pedestrians, two-wheelers, cars, buses, etc. sharing the road. We use the tracking-by-detection approach where we track a road-agent by matching the appearance or bounding box region in the current frame with the predicted bounding box region propagated from the previous frame. Roadtrack uses a novel motion model called the Simultaneous Collision Avoidance and Interaction (SimCAI) model to predict the motion of road-agents by modeling collision avoidance and interactions between the road-agents for the next frame. 
We demonstrate the advantage of RoadTrack on a dataset of dense traffic videos and observe an accuracy of 75.8\% on this dataset, outperforming prior \sota~tracking algorithms by at least 5.2\%. RoadTrack operates in realtime at approximately 30 fps and is at least 4$\times$ faster than prior tracking algorithms on standard tracking datasets.
\vspace{-5pt}
\end{abstract}

\IEEEpeerreviewmaketitle

\section{Introduction}
\label{sec1}
Tracking of road-agents on a highway or an urban road is an important problem in autonomous driving~\cite{teichman2011practical,ma2018autorvo} and related areas such as trajectory prediction~\cite{traPHic,chandra2019robusttp,chandra2019forecasting}. These road-agents may correspond to large or small cars, buses, bicycles, rickshaws, pedestrians, moving carts, etc. Different agents have different shapes, move at varying speeds, and their underlying dynamics constraints govern their trajectories. Furthermore, the traffic patterns or behaviors can vary considerably between highway traffic, sparse urban traffic, and dense urban traffic with a variety of such heterogeneous agents, \textit{e.g.,} in Figure~\ref{cover}. The traffic density can be defined based on the number of distinct road-agents captured in a single frame of the video or the number of agents per unit length of the roadway.

Given a traffic video, the tracking problem corresponds to computing the consistency in the temporal and spatial identity of all agents in the video sequence. Recent developments in autonomous driving and large-scale deployment of high-resolution cameras for surveillance has generated interest in the development of accurate tracking algorithms, especially in dense scenarios with a large number of heterogeneous agents. The complexity of tracking increases in dense scenarios as different types of road-agents come in close proximity and interact with each other. 
Examples of such interactions include passengers boarding or deboarding buses, bicyclists riding alongside cars and so on. Such traffic scenarios arise frequently in densely populated metropolitan cities.

\begin{figure}
\centering
\includegraphics[width = .8\columnwidth]{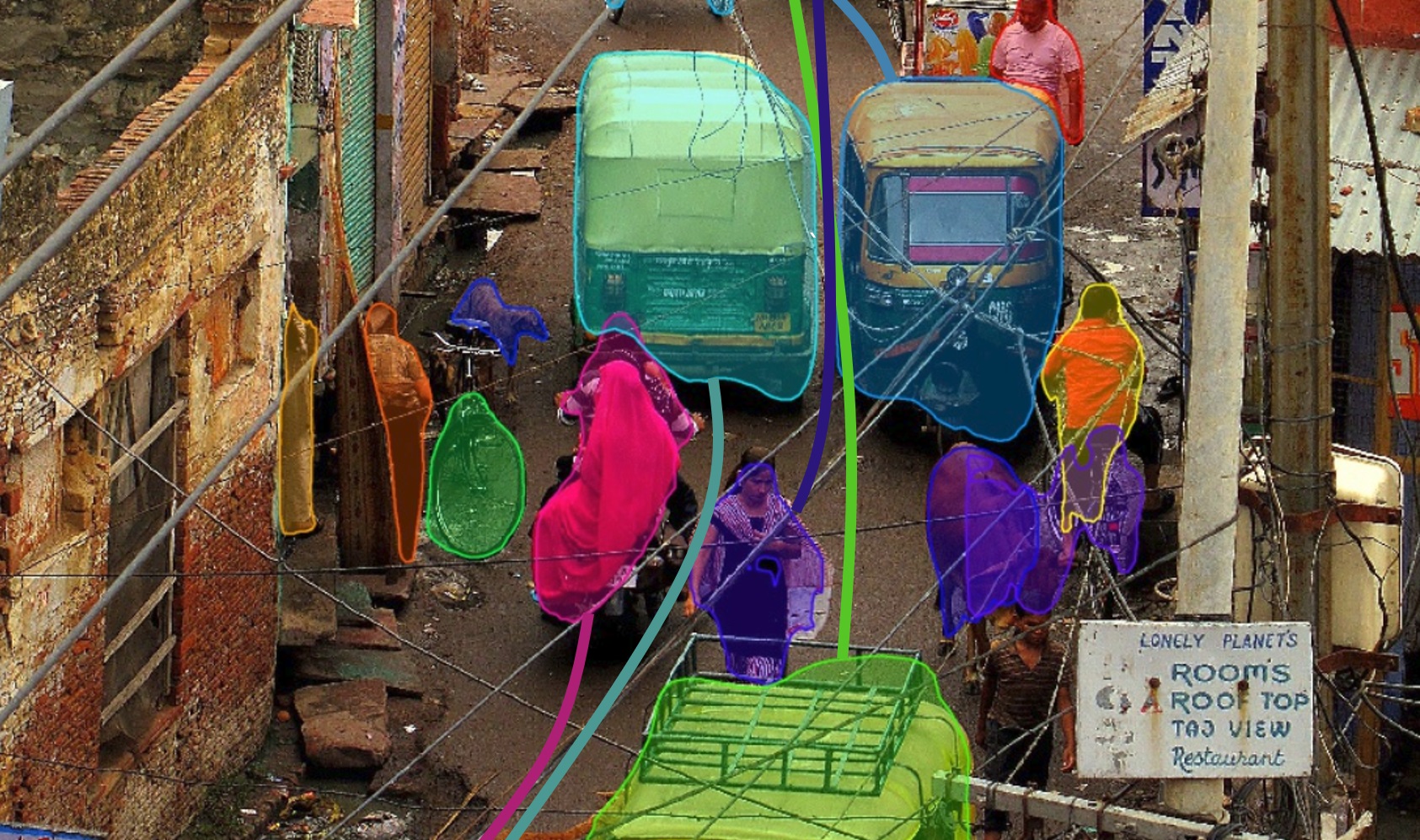}
  \caption{We highlight the performance of our tracking algorithm, RoadTrack, in this urban video. This frame consists of 27 road-agents, including pedestrians, two-wheel scooters, three-wheel rickshaws, cars, and bicycles. RoadTrack can track agents with 75.8\% accuracy at approximately 30 fps on an TitanXp GPU. We observe average improvement of at least $5.2$\% in MOTA accuracy and 4$\times$ in fps over prior methods.} 
  \label{cover}
  \vspace{-18pt}
\end{figure}

There is extensive prior work on tracking objects and road-agents~\cite{tracking1,tracking2}. But they are mostly designed and used in scenarios with sparse or lower density of road-agents. Such methods are unable to perform tracking in dense traffic due to occlusions and other challenges. In this paper, we mainly focus on developing efficient algorithms for dense traffic scenarios with heterogeneous interactions.

Recently, techniques based on deep learning are widely used for object detection and tracking.
In order  to solve the tracking problem in dense and heterogeneous traffic scenarios, we require a motion model that can account for interactions among heterogeneous agents and the high density in which these agents move. We adopt the tracking-by-detection paradigm, which is a two-step process of object detection and state prediction using the motion model. The first step, object detection, is performed to generate vectorized representations, called features, for each road-agent that facilitate identity association across frames. The second step is to predict the state (position and velocity) for the next frame using a motion model.

\textbf{Main Contributions:}  We present a realtime tracking algorithm, called RoadTrack, to track heterogeneous road-agents in dense videos. RoadTrack uses a new motion model to represent the motion of different road-agents by simultaneously accounting for collision avoidance and pairwise interactions. We show it is better suited for dense and heterogeneous traffic scenes in comparison to linear constant velocity, non-linear, and learning-based motion models. We name this motion model, SimCAI (``Simultaneous Collision Avoidance and Interaction (SimCAI)").

RoadTrack makes no assumption regarding camera motion and camera view. For example, we show our algorithm can track road-agents in heavy traffic captured from both front view and top view cameras that can be either stationary or moving. We further do not make assumptions for lighting conditions and can even track road-agents during night-time with glare from oncoming traffic (see supplementary video).

\textbf{Main Benefits: }The advantages of using RoadTrack are summarized below:
\begin{enumerate}
    \item \textbf{Accuracy:} On a dense traffic dataset, RoadTrack is \sota with an absolute accuracy of 75.8\%. This is an increase of 5.2\% over the next best method. This is equivalent to a rank difference of 42 with the next best method on the current \sota tracking benchmark dataset~\cite{mot16}. 
    
    \item \textbf{Speed:} Our method demonstrates realtime performance at approximately 30 fps on dense traffic scenes containing up to 100 agents per frame as well as standard tracking datasets. All results were obtained on a TITAN Xp GPU with 8 cores of CPU at 3.6 Ghz frequency. On the MOT benchmark, RoadTrack is at least 4$\times$ faster than SOTA methods, and on the dense traffic dataset, it is comparable to the fastest SOTA method. 
\end{enumerate}
\vspace{-10pt}



\section{Related Work}
\label{sec2}
\subsection{Pedestrian and Vehicle Tracking}
There is extensive work on pedestrian tracking \cite{cui2005tracking,kratz_tracking_2011}. Bruce et al. \cite{bruce2004better} and Gong et al. \cite{gong} predict pedestrians' motions by estimating their destinations. Liao et al.~\cite{liao2003voronoi} compute a Voronoi graph from the environment and predict the pedestrian motion along the edges. Mehran et al.~\cite{mehran_abnormal_2009} apply the social force model to detect anomalous pedestrian behaviors from videos. Pellegrini et al.~\cite{pellegrini2009you} use an energy function to build a goal-directed short-term collision-avoidance motion model. Bera et al.~\cite{berareach, bera2016glmp} use reciprocal velocity obstacles and hybrid motion models to improve the accuracy. All these methods are specifically designed for tracking pedestrian movement.

Vehicle tracking has been studied in computer vision, robotics, and intelligent transportation. Some of the earlier techniques are based on using cameras~\cite{dellaert1997robust} and laser range finders~\cite{streller2002vehicle}. The authors of \cite{petrovskaya2009model} model dynamic and geometric properties of the tracked vehicles and estimate their positions using a stereo rig mounted on a mobile platform. Ess et al.~\cite{ess2010object} present an approach to detect and track vehicles in highly dynamic environments. Multiple cameras have also been used to perform tracking all surrounding vehicles~\cite{rangesh2018no,darms2008classification}. Moras et al.~\cite{moras2011credibilist} use an occupancy grid framework to manage different sources of uncertainty for more efficient vehicle tracking; Wojke et al.~\cite{wojke2012moving} use LiDAR for moving vehicle detection and tracking in unstructured environments. Finally,~\cite{coifman1998real} uses a feature-based approach to track the vehicles under varying lighting conditions. Most of these methods focus on vehicle tracking and do not take into account interactions with other road-agents such as pedestrians, two-wheelers, rickshaws etc. in dense urban environments. For an up-to-date review of tracking-by-detection algorithms, we refer the reader to methods submitted to the MOT benchmark~\cite{mot16}.

\vspace{-5pt}
\subsection{Motion Models for Tracking}
There is substantial work on tracking multiple objects and use of motion models to improve the accuracy~\cite{luber2010people, mht-lin1,edmt-lin3,lfnf-lin4, chandra2019densepeds, bera2014adapt}. 
Kim et al.~\cite{mht-lin1} perform Multiple Hypotheses Tracking (MHT) ~\cite{mht} by using an effective online classifier for efficient branch pruning. The constant velocity linear motion model has been used to join fragmented pedestrian tracks caused by occlusion~\cite{lfnf-lin4}. However, dense traffic often cause road-agents to perform complex maneuvers to avoid collisions that are often non-linear. Hence, linear motion models do not work well in dense scenes.

RVO~\cite{van2011reciprocal} is a non-linear motion model that has been used for pedestrian tracking in dense crowd videos. However, RVO does not take into account agents interacting with one another. An extension to RVO, called AutoRVO~\cite{ma2018autorvo} includes dynamic constraints between road-agents. However, AutoRVO is based on CTMAT~\cite{ctmat} representations of road-agents that cannot be translated to front-view scenes. Other non-linear motion models that have been used for tracking include social forces~\cite{helbing1995social}, LTA~\cite{pellegrini2009you}, and ATTR~\cite{yamaguchi2011you}. However, these are mainly designed for tracking pedestrians. Social Forces, in particular, holds resemblance to our proposed motion model, SimCAI, in that it models the attraction and repulsion between agents (pedestrians only) through the concept of potential energy functions. With the recent rise in popularity of deep learning, recurrent neural networks such as LSTMs have been used as motion models for tracking~\cite{online152,rtdl5-online153}. We compare SimCAI with both learning- and non-learning-based motion models in this paper.


     
\section{RoadTrack: Overview}
\label{sec4}

In this section, we present the RoadTrack algorithm that combines Mask R-CNN object segmentation with SimCAI. Informally, the tracking problem is stated as follows: Given a video, we want to assign an ID to all road-agents in all frames. This is formally equivalent to solving the following sub-problem at each time-step (or frame): At current time $t$, given the ID labels of all road-agents in the frame, assign labels for road-agents in the next frame (time $t+1$).  

We start by using Mask R-CNN to implicitly perform pixel-wise segmentation of the road-agents. This generates a set of segmented boxes~\cite{chandra2019densepeds}. For each detected road-agent, $h_j$, generated using Mask R-CNN, we extract their corresponding features, $\xhj{f}$, using the deep learning-based feature extraction architecture proposed in~\cite{chandra2019densepeds}. We do not use the provided pre-trained models and instead, fine-tune the existing feature extraction network on traffic datasets to learn meaningful features pertaining to traffic. We discuss the fine-tuned hyperparameters in the supplementary material. 

Next, we predict the next state (state consists of spatial coordinates ($p_i$) and velocities ($\nu_i$)) for each road-agent for the next time-step using SimCAI. This step is the main contribution of this work and is described in detail in Section~\ref{sec: simcai}. This step results in another set of segmented boxes for each road-agent at time $t+1$. 

Finally, we use these sets of segmented boxes to compute features using a Convolutional Neural Network~\cite{deepsort}. The features generated are compared using association algorithms~\cite{kuhn2010hungarian} to compute the ID of each agent in the next frame. The features are matched in two ways: the Cosine metric and the IoU overlap \cite{iou}. The Cosine metric is computed using the following optimization problem:

\begin{equation}
\min_{\h}(\mathnormal{l}(\xpi{f}, \xhj{f})| p_i \in \mc{P}, h_j \in \mc{H}_i).
\label{optim}
\end{equation}
where $\mc{H}_i$ is the subset of all detected road-agents in the current frame that are within a circular region around agent $p_i$ that have not been matched to a predicted agent. The IoU overlap metric is used in conjunction with the cosine metric. This metric builds a cost matrix $\Sigma$ to measure the amount of overlap of each predicted bounding box with all nearby detection bounding box candidates. $\Sigma(i,j)$ stores the IoU overlap of the bounding box of $\p$ with that of $\h$ and is calculated as:
\begin{equation*}
    \Sigma(i,j) = \dfrac{\bb{B}_{ \p} \cap \bb{B}_{\h}}{\bb{B}_{\p} \cup \bb{B}_{\h}},\h \in \mc{H}_i.
\end{equation*}
If we denote the cosine and the IOU overlap metrics by $C$ and $I$, respectively, then the combined cost function value is obtained through,
\begin{equation}
    \textrm{Combined Cost} = \lambda_1 C + \lambda_2 I, \lambda_1+\lambda_2 = 1,
\end{equation}
where $\lambda_1, \lambda_2$ are constants representing the weights for the individual metric costs. Matching a detection to a predicted measurement with maximum overlap thus becomes a max-weight matching problem and we solve it efficiently using the Hungarian algorithm~\cite{kuhn2010hungarian}. The ID of the road-agent at time $t$ is assigned to that road-agent at time $t+1$ whose appearance is most closely associated to the road-agent at time $t$.
\vspace{-5pt}

\section{SimCAI: Simultaneous Collision Avoidance and Interactions}
\label{sec: simcai}


One of the major challenges with tracking heterogeneous road-agents in dense traffic is that road-agents such as cars, buses, bicycles, road-agents, etc. have different sizes, geometric shape, maneuverability, behavior, and dynamics. This often leads to complex inter-agent interactions that have not been taken into account by prior multi-object trackers. 
Furthermore, road-agents in high-density scenarios are in close-proximity to one another or are almost colliding. So we need an efficient approach for predicting the next state of a road-agent by modeling the collisions and interactions. We thus present SimCAI, that takes into account both,
\begin{itemize}[noitemsep]
    \item Reciprocal collision avoidance~\cite{van2011reciprocal} with car-like kinematic constraints for trajectory prediction and collision avoidance.
    \item Heterogeneous road-agent interaction between pedestrians, two-wheelers, rickshaws, buses, cars and so on.
    \end{itemize}
    
All the notations used in the paper are provided in Table I of full version of this text~\cite{chandra2019roadtrack}. 

\subsection{Velocity Prediction by Modeling Collision Avoidance}
\label{sec3A}
Reciprocal Velocity Obstacles (RVO)~\cite{van2011reciprocal} extends Velocity Obstacles motion model by modeling collision avoidance behavior for multiple engaging agents. RVO can be applied to pedestrians in a crowd and we modify it to work with bounding boxes as our algorithm conforms to the tracking-by-detection paradigm.

We represent each agent as, $\Psi_{t} = [u,v,\Dot{u}, \Dot{v}, v_{\tim{pref}}]$,
where $u,v,\Dot{u},\dot{v}, \tim{and} \  v_{\tim{pref}}$ represent the top left corner of the bounding box, their velocities, and the preferred velocity of the agent in the absence of obstacles respectively. $v_{\tim{pref}}$ is computed internally by RVO. 



The computation of the new state, $\Psi_{t+1}$, is expressed as an optimization problem. For each agent, RVO computes a feasible region where it can move without collision. This region is defined according to the RVO collision avoidance constraints (or ORCA constraints~\cite{van2011reciprocal}). If the ORCA constraints forbid an agent's preferred velocity, that agent chooses the velocity  closest to its preferred velocity that lies in the feasible region, as given by the following optimization: 
\begin{equation}
    v_{\textrm{new}} = \argmin_{\substack{v \notin ORCA}} ||v - v_{\textrm{pref}}||
\end{equation}

The velocity, $v_{\textrm{new}}$, is then used to calculate the new position of a road-agent.

The difference in shapes, sizes, and aspect ratios of road-agents motivate the need to use appearance-based features. In order to combine object detection with RVO, we modify the state vector, $\Psi_{t}$, to include bounding box information by setting the position to the centers of the bounding boxes. Thus, $u = { \frac{u+w}{2}}$ and $v ={ \frac{v+h}{2}}$, where $w,h$ denote the width and height, respectively, of the corresponding bounding box.

Finally, the original RVO models the motion of agents seen from a top-view. Therefore, to account for front-view traffic as well as top-view, we use the modification proposed by the authors of~\cite{chandra2019densepeds} that allow RVO to model the motion of road-agents in front-view traffic scenes.

\subsection{Velocity Prediction by Modeling Road-Agent Interactions}
\label{sec3B}
In a traffic scenario, interactions can occur between different types of road-agents: vehicle-vehicle, pedestrian-pedestrian, vehicle-pedestrian, bicycle-pedestrian, etc. In this section, we present a formulation to model such interactions. Our input is an RGB video captured from a camera with known camera parameters. By using the camera center as the origin, we transform pixel coordinates to scene coordinates for the computations that follow in this section.


\subsubsection{\textbf{Intent of Interaction}}
\label{3B1} 
The idea of using spatial regions to characterize agent behavior was proposed in~\cite{hall1966hidden}. The authors introduced the notion of ``public" and ``social" regions, that are of the form of concentric circles. We show a quadrant of these regions in Figure~\ref{hti}, where the yellow area is the social region and the orange area is the public region. Based on this work, Satake et al. \cite{satake2009approach} proposed a model of approach behavior with which a robot can interact with humans. At the public distance the robot is allowed to approach the human to interact with them, and at the social distance, interaction occurs. In SimCAI, we have set the public and social distances heuristically. 

We say that a road-agent, $\p$, intends to interact with another agent, $p_k$, when $\p$ is within the social distance of $p_k$ for some minimum time $\tau$. When two road-agents intend to interact, they move towards each other and come in close proximity. 
    
\subsubsection{\textbf{Ability to Interact}}
\label{3B2}
Even when two road-agents want to interact, their movements could be restricted in dense traffic. We determine the ability to interact (Figure~\ref{hti}(right)) as follows. 

Each agent has a personal space, which we define as a circular region $\zeta$ of radius $\rho$, centered around $p_k$. Given a road-agent $\p$, the slope of its $v_{\tim{pref}}$ is $\tan{\theta}$. $\theta$ is the angle with the horizontal defined in the world coordinate system. In dense traffic, each agent, $\p$ has a limited space in which they can steer, or turn. This space is the feasible region determined by the ORCA constraints described in the previous section. We define a 2D cone, $\gamma$, of angle $\phi$ as the ORCA region in which the agent can steer. $\phi$ is thus the steering angle of the agent. We denote the extreme rays of the cone as $r_1$ and $r_2$. $\bot_{\bb{G_1}}^{\bb{G_2}}$ denotes the smallest perpendicular distance between any two geometric structures, say, $\bb{G_1}$ and $\bb{G_2}$. These parameters are fixed for different agent types and are not learned from data. 

If $\p$ has intended to interact with $p_k$, the projected cone of $\p$, defined by extending $r_1$ and $r_2$, is directed towards $p_k$
Then, in order for interaction to take place, it is sufficient to check for either one of two conditions to be true:

\begin{enumerate}[noitemsep]
    \item Condition $\Omega_1$: Intersection of $\zeta$ with either $r_1$ or $r_2$ (if either ray intersects, then the entire cone intersects $\zeta$).
    \item Condition $\Omega_2$: $\zeta \subset \gamma$ (if $\zeta$ lies in the interior of the cone, see Figure~\ref{hti}). 
\end{enumerate} 

\noindent For these conditions to hold, we require that the cone does not intersect or contain any $p_j \in \mc{P}, j \neq i$. We now make these equations more explicit.

We parametrize $r_1,r_2$ by their slopes $\tan \delta$, where $\delta = \theta_i + \phi_i$ if $\bot_\zeta^{r_1} \geq \bot_\zeta^{r_2}$, else $\delta = \theta_i - \phi_i$.
\begin{figure}[t]
    \centering
    \includegraphics[width = .9\columnwidth]{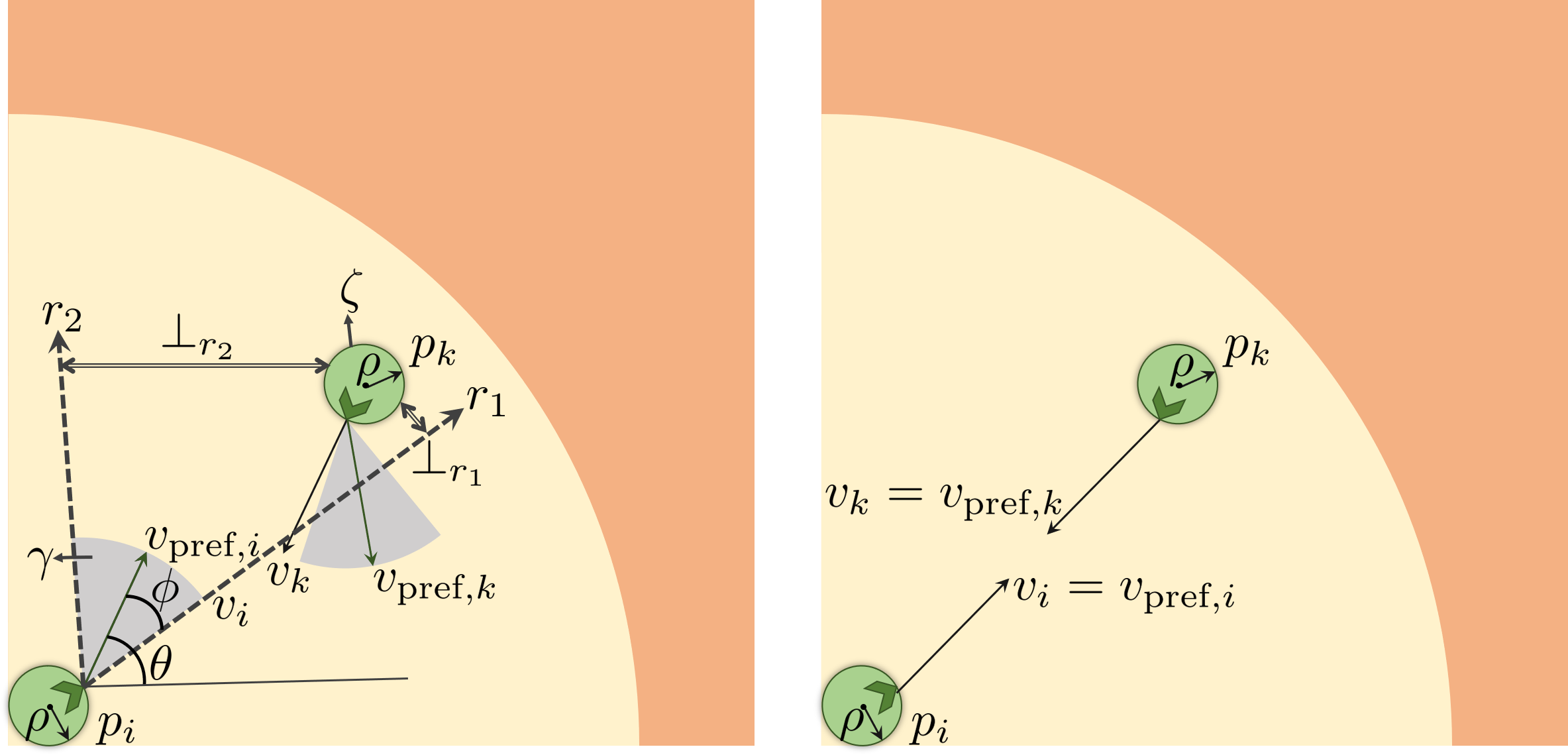}
    \caption{Inner yellow circle denotes the social distance and the outer orange area denotes the public region. At time $t \geq \tau$ and using \ref{3B1}, $\p$ intends to interact with $p_k$. Then using \ref{3B2} (left), $\p$ determines its ability to interact with $p_k$. We observe that $\gamma$ (grey cone) of $\p$ contains $\zeta$ of $p_k$ (green circle around $p_k$). Thus $\p$ can interact with $p_k$. Using \ref{3B3} (right), $\p$ and $p_k$ align their preferred velocities toward each other.}
    \label{hti}

\vspace{-15pt}
\end{figure}
The resulting equation of $r_1$ (or $r_2$) is $(Y - v_i) = \tan\delta(X - u_i)$ and the equation of $\zeta$ is $(X - u_k)^2 + (Y - v_k)^2 = \rho^2$. Solving both equations simultaneously, we obtain an equation $\Omega_1$.
Intersection occurs if the discriminant of $\Omega_1 \geq 0$. This provides us with the first condition necessary for the occurrence of an interaction between $\p$ and $p_k$.

Next, we observe that if $\zeta$ lies in the interior of $\gamma$, then $p_k$ lies on the opposite sides of $r_1$ and $r_2$ which is modeled by the following equation:
\begin{equation}
   \Omega_2 \equiv r_1(p_k).r_2(p_k) \leq 0
    \label{omega2}
\end{equation}
Solving Equation~\ref{omega2} further provides us with the second condition for the occurrence of an interaction between $\p$ and $p_k$,
where $\Omega_1, \Omega_2:\bb{R}^2 \times \bb{R}^2 \times \bb{R} \times \bb{R} \longmapsto \bb{R}$.

\subsubsection{\textbf{Interaction}}
\label{3B3}
If either $\Omega_1$ or $\Omega_2$ is true, then road-agents $\p, p_k$ will move towards each other to interact at time $t \geq \tau$. When this happens, we assume that $\p$ and $p_k$ align their current velocities towards each other. Thus, $v_\textrm{new} = v_\textrm{pref}$. The time taken for the two road-agents to be meet or converge with each other is given by $t = \dfrac{|| p_i - p_k ||_2}{|| v_i - v_k ||_2}$. If two road-agents are overlapping (based on the values of $\Omega_1$ and $\Omega_2$), we model them as a new agent with radius $2\epsilon$. 


Our approach can be extended to model multiple interactions. Currently, we restrict an interaction to take place between 2 road-agents. Therefore, in the case of multiple possible interactions with an agent, $p_k$, we form a set $\mc{Q} \subseteq \mc{P}$, where $\mc{Q}$ is the set of all road-agents $p_\omega$, that are intending to interact with $p_k$. We determine the road-agent that will interact with $p_k$ as the road-agent that minimizes the distance between $p_k$ and $p_\omega$ after a fixed time-step, $\Delta t$. Thus, $p_\omega = \argmin_{w} \lVert (p_\omega + v_\omega \Delta t) - p_k \rVert, p_\omega \in \mc{Q}$. road-agents that are not interacting avoid each other and continue moving towards their destination.

\subsection{Analysis}
\newcommand{\bigO}[1]{\mathcal{O}{#1}}
\label{sec3C}
We analyze the accuracy and runtime performance of SimCAI in traffic scenarios with increasing density and heterogeneity. 

\textbf{Accuracy Analysis:} We analytically show the advantage of SimCAI over other motion models such as Social Forces \cite{helbing1995social}, RVO \cite{van2011reciprocal, bera2014realtime}, and constant velocity~\cite{deepsort}. 


We denote the mutliple object tracking accuracy, $MOTA$ of a system using a particular motion model as ${MOTA}^{\rm{model}}$ and define it as ${MOTA}^{\rm{model}} = \sum_c {MOTA}_c + \sum_i {MOTA}_i$ where $c$ and $i$ denote an agent whose motion is being modeled using collision avoidance and interaction, and ${MOTA}_c$ and ${MOTA}_i$ denote their individual accuracies, respectively. Let $n$ represent the number of total road-agents in a video, then we have $n = n_c + n_i$, where $n_c, n_i$ correspond to the number of agents that are avoidaing collisions and are interacting, respectively.




Increasing $n$ would increase the number of road-agents whose motion is modeled through collision avoidance or heterogeneous interaction formulations. Linear models do not account for either formulation. Standard RVO only accounts for collision avoidance. SimCAI models both. Therefore, we rationalize that,

\begin{equation*} \label{eq1}
\scalebox{0.7}{
\begin{math}
\begin{aligned}
       & {MOTA}_c^{\rm{linear}} \leq {MOTA}_c^{\rm{RVO}} \approx {MOTA}_c^{\rm{SimCAI}} \\
      &{MOTA}_i^{\rm{linear}} \leq {MOTA}_i^{\rm{RVO}} \leq {MOTA}_i^{\rm{SimCAI}}\\
    \implies &{MOTA}^{\rm{linear}} \leq {MOTA}^{\rm{RVO}} \leq {MOTA}^{\rm{SimCAI}}
\end{aligned}
\end{math}
}
\end{equation*}

We validate the analysis presented here in Section \ref{subsec:results}.

\begin{figure}
    \centering
    \includegraphics[width=\columnwidth]{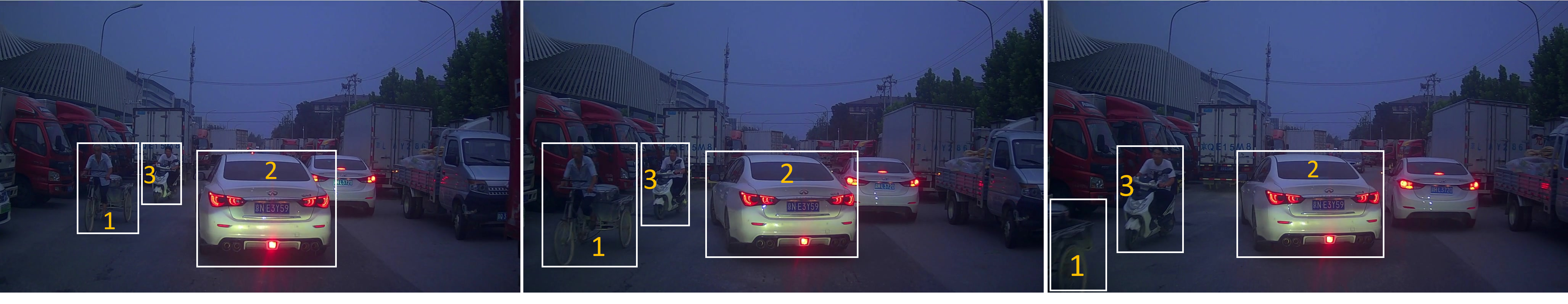}
    \caption{Qualitative analysis of RoadTrack on the TRAF dataset at night time consisting of cars, 2-wheelers, 3-wheelers, and trucks. Frames are chosen with a gap of 2 seconds($\sim$ 60 frames). For visual clarity, each road-agent is associated with a unique ID number. The ID is displayed in orange. Note the consistencies in the ID, for example, the 3-wheeler (1), car (2), and 2-wheeler (3).}
    \label{fig:traf23image}
    \vspace{-15pt}
\end{figure}

\textbf{Runtime Analysis:} At approximately 30 fps, we achieve a minimum speed-up of approximately 4$\times$, and upto approximately 30$\times$, over \sota methods on the MOT dataset~(Table~\ref{table3}). The selection of \sota methods is done in Section~\ref{subsec:methods}. The \sota use RNNs to model the motion of road-agents~\cite{online152,rtdl5-online153}, while we use the modified RVO formulation. We exploit the geometrical formulation of SimCAI to state and prove the following theorem:

\begin{thm}
Given $\mc{P} = \{ \p | 1 \leq i \leq n \}$, that represents a set of $n$ road-agents in a traffic scene that may assume any shape, size, and agent-type, if $\textrm{state}_{\p}$ $\in$ \{ \emph{stationary, collision avoiding, interacting} \}, $\forall i \in n$, then SimCAI can track the $n$ road-agents in $\bigO{(n_c + \omega n_i)}$, where $\omega << n_i$.
\label{theorem}
\end{thm}

\begin{proof}
RVO is based on linear programming and can perform tracking with a proven runtime complexity of $\bigO{(n)}$~\cite{van2011reciprocal}. Now, if we assume that agents always assume one of the following states: stationary, avoiding collision, or interacting, then we have $n = n_c + n_i$, where $n_c, n_i$ correspond to the number of agents in collision avoidance states and interacting states, respectively. We ignore stationary road-agents. Following the formulation in Section~\ref{sec3B}, for each interacting road-agent, SimCAI predicts a new velocity by solving a linear optimization problem over $\omega$ road-agents. Thus, the runtime complexity of SimCAI is $\bigO{(n_c + \omega n_i)}$, where $\omega << n_i$.
\end{proof}


Our high fps is a consequence of our linear runtime complexity and we validate our theoretical claims in Section~\ref{sec5}. We further hypothesize that prior deep learning-based methods~\cite{online152,rtdl5-online153} are less optimal in terms of runtime due to the intensive computation requirements by deep neural networks~\cite{lstmruntime,elstm}. For example, ResNet~\cite{resnet} needs more than 25 MB for storing the computed model in memory, and more than 4 billion float point operations (FLOPs) to process a single image of size 224$\times$224~\cite{lstmruntime}.

We would like to clarify that by realtime performance, we refer to the realtime computation of the tracking algorithm only. We do not consider the computation time of Mask R-CNN. This is standard practice by tracking-by-detection algorithms~\cite{rtdl5-online153} that only contribute to the tracking component, similar to this work. We therefore compare with realtime tracking algorithms.

\section{Experiments}
\label{sec5}


\subsection{Datasets}
We highlight the performance of RoadTrack through extensive experiments on different traffic datasets.
\textbf{(Dense) TRAF Dataset: }We use the TRAF traffic dataset~\cite{traPHic} that consists of a set of $60$ video sequences that contain dense traffic with highly heterogeneous agents with front and top-down viewpoints, stationary and moving camera motions, and during both day and night. These videos are of highway and urban traffic in high population countries like China and India. Most importantly, ground truth annotations consisting of 2-D bounding box coordinates and agent types are provided with the dataset. The key aspects of this dataset are the high density and the heterogeneity. 

\textbf{(Sparse) MOT \& KITTI-16 Datasets: }There are now several popular open-source tracking benchmarks available on which researchers can test and compare the performance of tracking algorithms. The current \sota~ benchmark is the MOT benchmark~\cite{mot16}, which contains a mix of pedestrians and traffic sequences. However, the MOT benchmark is a general tracking benchmark dataset. Therefore, we additionally conduct experiments exclusively on the KITTI-16 traffic sequence~\cite{kitti}. It should be noted that the KITTI-16 sequence is sparse, consisting of mostly cars, and does not contain road-agent interactions.


\subsection{Evaluation Methods}
\label{subsec:methods}
Due to the open-source nature of the MOT benchmark, there are a large number of methods available in the MOT benchmark~(80 and 87 on 2D MOT15 and MOT16, respectively). To demonstrate the superiority of RoadTrack, it is therefore sufficient to select \sota methods from all the methods, and compare RoadTrack against this set of methods. We define a \sota method as one that satisfies all of the following criteria simultaneously:

\begin{enumerate}[noitemsep]
    \item Higher Average Rank: The MOT benchmark assigns an ``average rank'' to each method. The average rank (\cite{rank}, page 390) of a tracking algorithm is computed by averaging over all the metrics. This metric effectively ranks the overall performance of a tracking algorithm by taking all the metrics into account simultaneously. We select competitor methods that have a \textit{higher (better)} average rank than ours.
    \item Published Work: Many of the tracking methods submitted on the MOT benchmark are anonymous. We therefore select methods that are published in peer-reviewed conferences and journals.
    \item Online Tracking: RoadTrack performs tracking using only the information from the previous frame and assumes no knowledge of future frames, thus making it an online tracking method. Therefore, we compare RoadTrack with top-performing online methods for fair comparison. 
    \item Realtime Performance: RoadTrack has realtime performance and performs tracking at up to approximately 30 fps (see Tables~\ref{table3},\ref{tab:compare_on_MOT}). Note that by realtime performance, we refer to the realtime computation of the tracking algorithm only. We do not consider the computation time of Mask R-CNN. This is conventionally accepted by tracking-by-detection methods that optimize only the tracking component. We therefore compare with algorithms that also compute tracking in realtime.
\end{enumerate}

The methods that satisfy these criteria are listed in Tables~\ref{table3} and~\ref{tab:compare_on_MOT}. For evaluation on the TRAF dataset, however, one additional criterion is required: the availability of open-sourced code. There are only two methods (Table~\ref{comparison}) that satisfy all of the above criteria.

We point out that in selecting methods to compare with for each dataset according to the above criteria, all tables need not have the same selection of methods. For example, the methods in Tables~\ref{table3},\ref{tab:compare_on_MOT} do not have open-sourced code.

\subsection{Evaluation Metrics}

We use standard tracking metrics defined in~\cite{clear}. We compare the overall accuracy (MOTA) which is computed using the formula: $\textrm{MOTA}=\frac{1-(\textrm{FN}+\textrm{FP}+\textrm{IDS})}{\textrm{GT}}$ where FN, FP, IDS, and GT correspond to the number of false negatives, false positives, ID switches, and ground truth agents, respectively. Additionally, we report the number of mostly tracked (MT) and mostly lost (ML) agents as well as the precision of the detector (MOTP), as per their provided definitions in~\cite{clear}. In accordance with the strict annotation protocol adopted by the MOT benchmark, we do not count stationary agents such as parked vehicles in our formulation. Detected objects such as traffic signals are thus considered false positives. 

\begin{table}[!htb]
  \centering
  \resizebox{\columnwidth}{!}{%

  \begin{tabular}{lccccccccc}
  \toprule
    Dataset & Tracker & FPS$\uparrow$ & MT(\%)$\uparrow$ & ML(\%)$\downarrow$ & IDS$\downarrow$ & FN$\downarrow$ & MOTP(\%)$\uparrow$ & MOTA(\%)$\uparrow$ \\
    \midrule
    \multirow{3}{*}{TRAF1} & MOTDT & 37.9& 0 & 98.2 & \textbf{15 (\textless 0.1\%)} & 18,764 (33.0\%) & \textbf{63.3} & 67.0 \\
    & MDP & 9.3 & 0 & 98.2 & 21 (\textless 0.1\%)  & 18,667 (32.8\%) & 60.1 & 67.1 \\
    & \textbf{RoadTrack}& \textbf{43.9} & 0 & \textbf{95.6} & 163 (0.3\%) & \textbf{17,953 (31.6\%)} & 58.8 & \textbf{68.1} \\
    \cline{2-9}
    \multirow{3}{*}{TRAF2} & MOTDT \Tstrut & \textbf{41.6} & 0 & 98.8 & 17 (\textless 0.1\%) & 18,201 (32.7\%) & 60.3 & 67.3 \\
    & MDP & 20.9 & 0 & 100.0 & \textbf{7 (\textless 0.1\%)} & 18,105 (32.5\%) & 59.6 & 67.5 \\
    & \textbf{RoadTrack} & 12.3 & 0 & \textbf{92.3} & 55 (0.1\%) & \textbf{17,202 (30.9\%)} & \textbf{60.8} & \textbf{69.0} \\
    \cline{2-9}
    \multirow{3}{*}{TRAF3} & MOTDT \Tstrut & 50.7 & 3.3 & 67.1 & 64 (\textless 0.1\%)  & 34,883 (27.0\%) & 69.6 & 72.9 \\
    & MDP & \textbf{51.8} & 0 & 100.0 & \textbf{0 (0.0\%)} & 43,057 (33.3\%) & 69.2 & 66.7 \\
    & \textbf{RoadTrack} & 36.6 & \textbf{32.2} & \textbf{40.0} & 62 (\textless 0.1\%) & \textbf{19,521 (15.1\%)} & \textbf{70.1} & \textbf{84.8} \\
    \cline{2-9}
    \multirow{3}{*}{TRAF4} & MOTDT \Tstrut& 36.6 & 1.2 & 76.3 & 123 (0.1\%)  & 54,849 (29.0\%) & 65.3 & 70.9 \\
    & MDP & 9.0 & 1.2 & 87.2 & \textbf{16 (\textless 0.1\%)}  & 59,097 (31.3\%) & \textbf{66.2} & 68.7 \\
    & \textbf{RoadTrack} & \textbf{40.6} & \textbf{6.0} & \textbf{54.6} & 266 (0.1\%) & \textbf{47,444 (25.1\%)} & 65.1 & \textbf{74.7} \\
    \cline{2-9}
    \multirow{3}{*}{TRAF5} & MOTDT \Tstrut& 36.0 & 0.7 & 75.9 & 221 (0.2\%)  & 33,774 (28.9\%) & 63.2 & 70.9 \\
    & MDP & 22.5 & 0 & 98.4 & \textbf{6 (\textless 0.1\%)} & 38,091 (32.6\%) & \textbf{64.9} & 67.3 \\
    & \textbf{RoadTrack} & \textbf{41.4} & \textbf{1.5} & \textbf{55.7} & 299 (0.3\%) & \textbf{24,860 (21.3\%)} & 63.1 & \textbf{78.4} \\
    \cline{2-9}
    \multirow{3}{*}{TRAF6} & MOTDT \Tstrut& \textbf{33.0} & 0 & 87.5 & 161 (0.1\%) & 58,212 (29.4\%) & 63.3 & 70.5 \\
    & MDP & 4.3 & 0 & 99.3 & \textbf{0 (0.0\%)} & 65,687 (33.2\%) & \textbf{68.6} & 66.8 \\
    & \textbf{RoadTrack} & 14.6 & \textbf{0.7} & \textbf{67.8} & 283 (0.1\%) & \textbf{52,017 (26.3\%)} & 62.8 & \textbf{73.6} \\
\midrule
\multirow{3}{*}{Summary} \Tstrut & MOTDT & \textbf{34.7} & 0.9 & 83.6 & 601 (0.1\%)  & 218,683 (29.3\%) & 65.5 & 70.6 \\
    & MDP & 10.1 & 0.2 & 97.0 & \textbf{50 (\textless 0.1\%)} & 242,704 (32.6\%) & 65.3 & 67.4 \\
    & \textbf{RoadTrack} & 31.6 & \textbf{7.0} & \textbf{66.9} & 1128 (0.2\%)  & \textbf{178,997 (24.0\%)} & \textbf{65.7} & \textbf{75.8} \\
    \bottomrule
  \end{tabular}
  }

  \caption{Evaluation on the TRAF dataset with MOTDT~\cite{rtdl3} and MDP \cite{xiang2015learning}. MOTDT is currently the best \textit{online} tracker on the MOT benchmark with open-sourced code. Bold is best. Arrows ($\uparrow, \downarrow$) indicate the direction of better performance. \textbf{Observation:} RoadTrack improves the accuracy (MOTA) over the state-of-the-art by 5.2\% and precision (MOTP) by 0.2\%.}
  \label{comparison}
\end{table}


 




 

\begin{table}[t]
  \centering
 \resizebox{\columnwidth}{!}{
 \begin{tabular}{clccccccc}
\toprule
& Tracker & FPS$\uparrow$  & MT(\%)$\uparrow$  & ML(\%)$\downarrow$ & IDS$\downarrow$  & FN$\downarrow$ & MOTP(\%)$\uparrow$ & MOTA(\%)$\uparrow$                \\ 
\midrule

& AP\_HWDPL\_p \Tstrut \cite{online151} & 6.7 & 17.6 & 11.8& 18&831   & \textbf{72.6} & 40.7 \\
 
\multirow{3}{*}{\rotatebox{90}{KITTI-16}}& RAR\_15\_pub \cite{rtdl5-online153} & 5.4  & 0.0  & 17.6& 18 &809  & 70.9 & 41.2                                   \\

& AMIR15 \cite{online152}& 1.9& 11.8& 11.8& 18 &714& 71.7 & \textbf{50.4 }             \\

& HybridDAT \cite{online154}& 4.6& 5.9& 17.6& \textbf{10}&706& \textbf{72.6} & 46.3            \\ 

& AM \cite{rtdl4} &  0.5 & 5.9 & 17.6& 19  &805& 70.5 & 40.6      \\
& \textbf{RoadTrack} \Tstrut & \textbf{28.9}  & \textbf{29.4} & \textbf{11.7} &15 &\textbf{668}& 71.3&12.2 
\\\bottomrule

\end{tabular}
}
  \caption{Evaluation on the KITTI-16 dataset from the MOT benchmark with \textit{online methods} that have an average rank higher than ours. RoadTrack is at least approximately 4$\times$ faster than prior methods. While we do not outperform on the MOTA metric, we still achieve the highest MT, ML, FN, and MOTP. We analyze our MOTA performance in Section~\ref{subsec:results}. Bold is best. Arrows ($\uparrow, \downarrow$) indicate the direction of better performance. The values for all methods correspond to the KITTI-16 sequence specifically, and not the entire 2D MOT15 dataset.}
 \vspace{-8pt}
  \label{table3}
\end{table}

\begin{table}[!htb]
  \centering
  \resizebox{\columnwidth}{!}{%
  \begin{tabular}{clccccccc}
  \toprule
    & Tracker & FPS$\uparrow$ & MT(\%)$\uparrow$ & ML(\%)$\downarrow$ & IDS$\downarrow$ & FN$\downarrow$ & MOTP(\%)$\uparrow$ & MOTA(\%)$\uparrow$ \\
    \midrule
    \multirow{5}{*}{\rotatebox{90}{2D MOT15}} & AMIR15 \cite{online152} & 1.9 & 15.8 & \textbf{26.8} & 1026 & 29,397 & 71.7 & 37.6 \\
    & HybridDAT \cite{online154}& 4.6 & 11.4 & 42.2 & 358 &31,140 & 72.6 & 35.0 \\
    & AM \cite{rtdl4} & 0.5 & 11.4 & 43.4 & \textbf{348} &34,848 & 70.5 & 34.3 \\
    & AP\_HWDPL\_p \cite{online151} & 6.7 & 8.7 & 37.4 & 586 &33,203 & 72.6 & \textbf{38.5} \\
    & \textbf{RoadTrack} \Tstrut & \textbf{28.9} & \textbf{18.6} & 32.7 & 429 &\textbf{27,499} & \textbf{75.6} & 20.0 \\
    \midrule
    \multirow{7}{*}{\rotatebox{90}{MOT16}} & EAMTT\_pub \cite{eamtt} \Tstrut & 11.8 & 7.9 & 49.1 & 965 & 102,452 & 75.1 & 38.8 \\
    & RAR16pub \cite{rtdl5-online153} & 0.9 & 13.2 & 41.9 & 648 &  91,173 & 74.8 & 45.9 \\
    & STAM16 \cite{rtdl4} & 0.2 & 14.6 & 43.6 & \textbf{473} &91,117 & 74.9 & 46.0 \\
    & MOTDT \cite{rtdl3} & \textbf{20.6} & 15.2 & 38.3 & 792  &85,431 & 74.8 & \textbf{47.6} \\
    & AMIR \cite{online152} & 1.0 & 14.0 & 41.6 & 774  & 92,856 & \textbf{75.8} & 47.2 \\
    & \textbf{RoadTrack} \Tstrut & 18.8 & \textbf{20.3} & \textbf{36.1} & 722 &\textbf{78,413} & 75.5 & 40.9 \\
    \bottomrule
  \end{tabular}
  }
\caption{Evaluation on the full MOT benchmark. The full MOT dataset is sparse and is not a traffic-based dataset. RoadTrack is at least approximately 4$\times$ faster than previous methods. While we do not outperform on the MOTA metric, we still achieve the highest MT, ML (MOT16), FN, and MOTP(MOT15). We analyze our MOTA performance in Section~\ref{subsec:results}. Bold is best. Arrows ($\uparrow, \downarrow$) indicate the direction of better performance. }
  \label{tab:compare_on_MOT}
  \vspace{-5pt}
\end{table}






\begin{table}[!htb]
  \centering
  \resizebox{\columnwidth}{!}{
  \begin{tabular}{lccccccc}
  \toprule
   Motion Model  & FPS$\uparrow$ & MT(\%)$\uparrow$ & ML(\%)$\downarrow$& IDS$\downarrow$ & FN$\downarrow$& MOTP(\%)$\uparrow$  & MOTA(\%)$\uparrow$ \\    \midrule
     Const. Vel  & 30& 0.0 & 100 & \textbf{11}  & 247,738(33.3\&) & \textbf{66.3} & 66.7  \\

       SF & 30 & 0.1 & 98.6 & 147 &  246,528 (33.1\%) & 63.8 & 66.3  \\

       RVO & 30& 0.0 & 100 & 38  & 247,675 (33.2\%) & 63.8 & 66.9 \\
        \textbf{SimCAI} &\textbf{30} & \textbf{7.0} & \textbf{66.9} & 1128 & \textbf{178,997 (24.0\%)} & 65.7 & \textbf{75.8} \\

                \bottomrule

  \end{tabular}
  }
  \caption{Ablation experiments to show the advantage of SimCAI. We replace SimCAI with a constant velocity (Const Lin Vel) \cite{deepsort}, Social Forces (SF) \cite{helbing1995social}, and RVO motion model (RVO)\cite{van2011reciprocal}. The rest of the method is identical to the original method. All variations operate at similar fps of approximately 30 fps. Bold is best. Arrows ($\uparrow, \downarrow$) indicate the direction of better performance.}
  \label{table4}
  \vspace{-15pt}
\end{table}

\subsection{Analysis of Results \& Discussion}
\label{subsec:results}
\textbf{On Dense Datasets: }We provide results on the TRAF dataset using RoadTrack and demonstrate a \sota average MOTA of 75.8\% (Table~\ref{comparison}). The aim of this experiment is to highlight the advantage of our overall tracking algorithm in dense and heterogeneous traffic. We compare RoadTrack with methods (selected according to criteria established in Section~\ref{subsec:methods}) on the dense TRAF dataset in Table~\ref{comparison}. MOTDT \cite{rtdl3} and MDP \cite{xiang2015learning} are the only state-of-the-art methods with available open-source code. All methods are evaluated using a common set of detections obtained using Mask R-CNN. Compared to these methods, we improve upon MOTA by 5.2\% on absolute. This is roughly equivalent to a rank difference of 46 on the MOT benchmark. 

MOTDT is currently the fastest method (according to the selection criteria of Section~\ref{subsec:methods}) on the MOT16 benchmark. Our approach operates at realtime speeds upto approximately 30 fps and is comparable with MOTDT (Table~\ref{comparison}). Our realtime performance results from the runtime analysis from Section~\ref{sec3C} and theorem~\ref{theorem}.

Note that we observe an abnormally high number of identity switches compared to other methods; however, this is because prior methods mostly fail to maintain an agent's track for more than 20\% of their total visible time (near 100\% ML). Not being able to track road-agents for most of the time excludes those agents as possible candidates for IDS, thereby resulting in lower IDS for prior methods. Interestingly, the low IDS score for prior methods also contributes to their reasonably high MOTA score, despite near-failure to track agents in dense traffic.

\textbf{On Standard Benchmarks: }
In the interest of completeness and thorough evaluation, we also evaluate RoadTrack on sparser tracking datasets and present results on both traffic-only datasets (KITTI-16) in Table~\ref{table3} as well as datasets containing only pedestrians (MOT) in Table~\ref{tab:compare_on_MOT}. RoadTrack's main advantage is SimCAI, which is based on modeling collision avoidance and interactions. In the absence of one or both, we do not expect it demonstrate  superior performance over prior methods on the sparse KITTI-16 and MOT datasets. 
%
While not conclusive, we believe our low MOTA score on the 2D MOT15 and KITTI-16 may also be attributed to a high number of detections that are incorrectly classified as false positives. For instance, road-agents that are too distant to be manually labeled are not annotated in the ground truth sequence. We observed this to be true for the methods we compared with as well. Therefore, we exclude FP from the calculation of MOTA for all methods in the interest of fair evaluation. 

We note, however, that RoadTrack is least 4$\times$ faster on the KITTI-16 and 2D MOT15 datasets at approximately 30 fps (Tables~\ref{table3},\ref{tab:compare_on_MOT}). To explain the speed-up, we refer to theorem~\ref{theorem} and the runtime analysis presented in Section~\ref{sec3C}. We specially point to the 15$\times$ and 5$\times$ speed-up over learning-based tracking methods,~\cite{online152,rtdl5-online153} in Table~\ref{table3} which we attribute the linear time computation of SimCAI as opposed to the intensive computation required by deep learning models.

\textbf{Ablation Experiments:}
We highlight the advantages of SimCAI through ablation experiments in Table~\ref{table4}. The aim of these experiments is to isolate the benefit of SimCAI. We compare with the following variations of RoadTrack in which we replace our novel motion model SimCAI with standard and \sota motion models, while keeping the rest of the system untouched:
\begin{itemize}[noitemsep]
    \item Constant Linear Velocity (Const Lin Vel). We replace SimCAI with a constant velocity linear motion model~\cite{deepsort}.
    \item Social Forces (SF). We replace SimCAI with the Social Forces motion model \cite{helbing1995social}.
    \item Reciprocal Velocity Obstacles (RVO) \cite{van2011reciprocal}. We replace SimCAI with the RVO motion model.
\end{itemize}

We compare SimCAI with other motion models (Constant linear velocity, Social Forces, and RVO) on the dense TRAF dataset. These experiments were performed by \textit{only} replacing SimCAI with other motion models, keeping the rest of the system unchanged. We observe that SimCAI outperforms the motion models by at least 8.9\% on absolute on MOTA. All the variations used in the ablation experiments operated at the same fps of approximately 30 fps. Additionally, we experimentally verify the analysis of Section~\ref{sec3C} by observing that ${MOTA}^{\rm{linear}} \leq {MOTA}^{\rm{RVO}} \leq {MOTA}^{\rm{SimCAI}}$. 
Once again, we point to our high IDS in Table~\ref{table4}, compared to the IDS of other motion models. As mentioned previously, this is due to the near-failure of other motion models (near 100\% ML) to track road agents in dense traffic. Not being able to track a road-agent excludes them as a IDS candidate.

\vspace{-5pt}
\section{Limitations and Future Work}
\label{sec6}
\vspace{-5pt}

There are many avenues of future work for our presented work. Currently, many parameters in our algorithm such as the radii for the social and public regions, steering angles, and cone angles, are heuristically chosen for optimum performance. It would be more efficient to learn these parameters instead, using data driven and machine learning techniques. Furthermore, the results from tracking road-agents can be directly used to further research in related areas such as trajectory prediction. With the increased popularity of deep-learning and improved tracking methods, deep learning techniques can be employed for predicting the future motion of road-agents in dense and heterogeneous traffic. 
\section{Acknowledgements}

This work was supported in part by ARO Grants W911NF1910069 and W911NF1910315, Semiconductor Research Corporation (SRC), and Intel.


\bibliographystyle{unsrt}
\bibliography{references}



\end{document}